\documentclass[sigconf,authorversion]{acmart}

\usepackage{amsmath,amsthm,mathtools}
\usepackage{graphicx} 
\usepackage{xspace}
\usepackage{color}
\usepackage[linesnumbered,ruled]{algorithm2e}

\usepackage{tikz}
\usetikzlibrary{decorations.pathreplacing}

\AtBeginDocument{%
  }

\newcommand{\oea}{$(1 + 1)$~EA\xspace}
\newcommand{\olea}{$(1 + \lambda)$~EA\xspace}
\newcommand{\ollga}{$(1 + (\lambda, \lambda))$~GA\xspace}
\newcommand{\onemax}{\textsc{OneMax}\xspace}

\DeclareMathOperator*{\argmax}{arg\,max}
\DeclareMathOperator{\Bin}{Bin}
\DeclareMathOperator{\Geom}{Geom}

\copyrightyear{2026}
\acmYear{2026}
\setcopyright{cc}
\setcctype{by}
\acmConference[GECCO '26]{Genetic and Evolutionary Computation Conference}{July 13--17, 2026}{San Jose, Costa Rica}
\acmBooktitle{Genetic and Evolutionary Computation Conference (GECCO '26), July 13--17, 2026, San Jose, Costa Rica}
\acmDOI{10.1145/3795095.3805156}
\acmISBN{979-8-4007-2487-9/2026/07}

\begin{document}

\title[When Switching Algorithms Helps: A Theoretical Study of Online Algorithm Selection]{When Switching Algorithms Helps:\\A Theoretical Study of Online Algorithm Selection}

\author{Denis Antipov}
\orcid{0000-0001-7906-096X}
\affiliation{
  \institution{Sorbonne Universit\'e, CNRS, LIP6}
  \city{Paris}
  \country{France}}
\email{denis.antipov@lip6.fr}

\author{Carola Doerr}
\orcid{0000-0002-4981-3227}
\affiliation{
  \institution{Sorbonne Universit\'e, CNRS, LIP6}
  \city{Paris}
  \country{France}}
\email{carola.doerr@lip6.fr}

\begin{abstract}
  Online algorithm selection (OAS) aims to adapt the optimization process to changes in the fitness landscape and is expected to outperform any single algorithm from a given portfolio. Although this expectation is supported by numerous empirical studies, there are currently no theoretical results proving that OAS can yield asymptotic speedups (apart from some artificial examples for hyper-heuristics). Moreover, theory-based guidelines for when and how to switch between algorithms are largely missing.

  In this paper, we present the first theoretical example in which switching between two algorithms---the $(1+\lambda)$~EA and the $(1+(\lambda,\lambda))$~GA---solves the OneMax problem asymptotically faster than either algorithm used in isolation. We show that an appropriate choice of population sizes for the two algorithms allows the optimum to be reached in $O(n\log\log n)$ expected time, faster than the $\Theta(n\sqrt{\frac{\log n \log\log\log n}{\log\log n}})$ runtime of the best of these two algorithms with optimally tuned parameters.

  We first establish this bound under an idealized switching rule that changes from the $(1+\lambda)$ to the $(1+(\lambda,\lambda))$~GA at the optimal time. We then propose a realistic switching strategy that achieves the same performance. Our analysis combines fixed-start and fixed-target perspectives, illustrating how different algorithms dominate at different stages of the optimization process. This approach offers a promising path toward a deeper theoretical understanding of OAS.
\end{abstract}



\maketitle

\section{Introduction}


\emph{Online algorithm selection} (OAS for brevity) is an approach to optimization based on switching between several optimization algorithms from a pre-defined portfolio in order to maximize the performance.\footnote{OAS is sometimes referred to as \emph{dynamic algorithm selection} (dynAS), but we prefer to use different notation to avoid confusion with \emph{dynamic algorithm configuration} (DAC) that involves some learning prior to running an algorithm.} This approach is especially helpful for iterative random search heuristics such as evolutionary algorithms (EAs), since it is easy to switch from one algorithm to another between iterations~\cite{VermettenWBD20,KostovskaJVNWED22,VermettenWSH23}. The ultimate goal of OAS is to always use the algorithm that yields the fastest progress at each stage of optimization, since it must outperform all algorithms from the portfolio.

The main challenge of OAS lies in designing effective switching strategies. This includes detecting when the currently employed algorithm ceases to be effective and determining which algorithm from the portfolio is best suited to continue the search. In practice, these challenges are typically addressed using learning-based techniques. However, any theory-based guidelines on the how to switch between algorithms in OAS settings are still largely missing.

The closest area to OAS that theory of evolutionary computation has studied before is the hyper-heuristics (HHs). HHs are usually embedded into EAs to select one or more of its components in each iteration. This includes HHs that select the operator to create offspring~\cite{AlanaziL14,LissovoiOW17,DoerrLOW18,LissovoiOW20}, parent selection strategy~\cite{CorusOZ25} or offspring selection strategy~\cite{LissovoiOW23,BendahiDFL25}. Most HHs use a simple strategy to switch between low-level heuristics, such as iteration through their portfolio in randomly shuffled order, using each heuristic from portfolio as long as it is successful, etc. Some of them involve more complex mechanisms, for example, they use a heuristic for some time $\tau$ to learn its usefulness~\cite{DoerrLOW18} or use a Markov chain to transition between different heuristics~\cite{BendahiDFL25}. The latter two HHs are very close to OAS, where the goal is to \emph{learn} the usefulness of each algorithm from the portfolio and to switch between them at the best moments.

\textbf{Our results.} In this paper, we address the lack of theoretical foundations for OAS by performing runtime analysis of an OAS algorithm that can switch between the \olea and the \ollga. The idea of such portfolio is inspired by the recent result of~\citet{AntipovBD25telo} who showed that the \ollga is much more efficient in the late stages of optimization on \onemax than any mutation-only EA and that unlike them it does not encounter the coupon collector effect to the same extent. At the same time, in early stages of optimization the \ollga wastes too many fitness evaluations on just a small progress that can be achieved with mutation. This implies that it makes sense to use the \olea (or even the \oea) in early stages of optimization, and switch to the \ollga later. We show that this OAS approach allows to find the optimum of \onemax in $O(n\log\log n)$ fitness evaluations. This is asymptotically faster than the $\Omega(n\log n)$ bound of all mutation-only EAs~\cite{LehreW12} and $\Theta(n\sqrt{\frac{\log n \log\log\log n}{\log\log n}})$ runtime of the \ollga with the best static parameters~\cite{DoerrD18}. This is the first proven example of OAS, where we observe that switching between two different algorithms brings performance that is asymptotically better than the performance of the best of these two algorithms. 

We first prove this runtime bound for the algorithm that knows its distance from the optimum and can decide to switch based on this information. Interestingly, the range of distances at which it can switch to achieve the optimal runtime is quite large, namely it spreads over $\Theta(\frac{n(\log\log n)^2}{\log n})$ different distances, which gives the OAS algorithm a lot of freedom to decide when to do it. To prove this result, we use the fixed-start analysis results for the \ollga from~\cite{AntipovBD25telo}, and also provide the fixed-target bounds for the \olea. 

Since in practice it is not realistic to have access to the current distance to the optimum, we provide a mechanism that detects the optimal time to switch from the \oea to the \ollga. This mechanism is based on the budget requirements of these algorithms, and it decides to switch when the \oea needs more fitness evaluations to get progress than it takes to run one iteration of the \ollga. We show that the OAS algorithm that embeds this mechanism is also capable of solving \onemax in $O(n\log\log n)$ time. To address the vague borderline between OAS and HHs, we also propose an HH that in each iteration decides whether it should perform an iteration of the \oea or of the \ollga. This can be considered as a choice between two operators that create offspring, but more complicated than in previous studies of HHs. We propose a strategy of how the HH can make this choice that is based on balancing the computation budget between the \oea and the \ollga, and we show that this HH can also achieve $O(n\log\log n)$ runtime. Apart from the result for the artificial GapPath function from~\cite{LehreO13,LissovoiOW17}, this is the first example when an HH that chooses operators to create an offspring has an asymptotical speed-up compared to the low-level heuristics it uses. 

\textbf{Synopsis.} The rest of the paper is organized as follows. In Section~\ref{sec:prelims}, we formally define the problem and the algorithms we analyze. In Section~\ref{sec:fixed-target}, we provide fixed-target results for the \olea that help us in our analysis. Then we perform runtime analysis of an OAS that has access to the distance to the optimum in Section~\ref{sec:optimal-switch}, and propose a switching mechanism that is independent of this distance in Section~\ref{sec:switching-moment}, where we also show that this mechanism has asymptotically the same performance. We translate our results into hyper-heuristics domain in Section~\ref{sec:hh}. We conclude the paper with a short discussion.

\section{Preliminaries}
\label{sec:prelims}

\subsection{Online Algorithm Selection}

Formally, online algorithm selection can be defined as a technique of running several algorithms sequentially, where each algorithm is chosen based on the information obtained by the previous algorithms, and this information is also passed to the algorithm. In the context of evolutionary optimization, this information is usually passed through the population of solutions, so that each next algorithm is initialized with the final population of the previous one. In this paper, we consider EAs with a trivial population, that is, population that consists of a single individual. In this setting, a switch to another algorithm can be done after any iteration, and the algorithm selector can be described by the pseudocode in Algorithm~\ref{alg:selector}.

\begin{algorithm}
    $x' \gets$ random solution from the search space\;
    \While{not terminated}{
        Select algorithm $A$ from portfolio $\mathcal{A}$\;
        Initialize $A$ with $x'$\;\label{line:init}
        Run $A$ optimizing function $f$ until some stopping criterion\;
        $x' \gets$ best solution found by $A$\;
    }
    \caption{An OAS algorithm optimizing function $f$.}
    \label{alg:selector}
\end{algorithm}

We note that in this pseudocode we do not specify multiple details. First, we do not specify the stopping criterion for the algorithm selector, which is typical for most runtime analysis results, but we assume that this stopping criterion is not satisfied before the algorithm finds the optimum. Second, we do not specify neither how we select algorithms, nor how we stop them, which will be discussed in the later sections.

\subsection{Available Algorithms}

In this subsection, we describe the algorithms which we are allowed to choose from in the algorithm selector. All these algorithms are designed for pseudo-Boolean optimization, that is, the search space is $\{0, 1\}^n$---the set of bit strings of length $n$ ($n$ is the \emph{problem size}).

The first algorithm is the \olea, where $\lambda$ is an integer parameter. The algorithm stores one individual $x$ (which is initialized with $x'$ in the algorithm selector, see line~\ref{line:init} in Algorithm~\ref{alg:selector}). In each iteration it independently generates $\lambda$ individuals (where $\lambda$ is a fixed parameter of the algorithm), each by creating a copy of $x$ and independently flipping each bit in it with probability $1/n$ (this variation operator is called \emph{standard bit mutation}). The \olea then selects an offspring $y$ with the best \emph{fitness} (the value of the optimized function) among the $\lambda$ offspring, breaking ties in an arbitrary way (for example, uniformly at random). If the selected offspring $y$ is not worse than its parent $x$, it replaces the parent. One iteration of this algorithm costs $\lambda$ fitness evaluations. A pseudocode of this algorithm for \emph{maximization} is shown in Algorithm~\ref{alg:olea}.

\begin{algorithm}
    \textbf{input:} individual $x$\;
    \While{not terminated}{
        \For{$i \in [1..\lambda]$}{
            $y_i \gets$ copy of $x$\;
            Independently flip each bit in $y_i$ with prob. $1/n$\;
        }
        $y \gets \argmax_{z \in \{y_1, \dots, y_\lambda\}}f(z)$\;
        \If{$f(y) \ge f(x)$}{
            $x \gets y$\;
        }
    }
    \textbf{return} $x$\;
    \caption{The \olea maximizing a pseudo-Boolean function $f$.}
    \label{alg:olea}
\end{algorithm}

The second algorithm is the \ollga with a fixed parameter $\lambda$ proposed in~\cite{DoerrDE15}. It also stores only one individual $x$ that is initialized with $x'$ in line~\ref{line:init} of Algorithm~\ref{alg:selector}. Its iterations are performed in two steps. In the first step, the \ollga chooses a number $\ell$ following the binomial distribution $\Bin(n, \lambda/n)$. It then independently creates $\lambda$ offspring of $x$, each by flipping exactly $\ell$ bits that are chosen uniformly at random (note that this means that all offspring have the same distance $\ell$ from their parent $x$). The best of these $\lambda$ offspring is chosen as the mutation winner $x'$ (ties are broken uniformly at random). In the second step, the \ollga independently creates another $\lambda$ offspring, each by applying a biased crossover to $x$ and $x'$. This biased crossover takes each bit value from $x'$ with probability $1/\lambda$ and from $x$ with probability $1 - 1/\lambda$ (independently for each of $n$ positions). The best of these $\lambda$ offspring is considered to be the winner $y$, and if $y$ is not worse than $x$, the algorithm replaces $x$ with $y$. The pseudocode of the \ollga is shown in Algorithm~\ref{alg:ollga}. The cost of one iteration is $2\lambda$ fitness evaluations. Note that parameter $\lambda$ can only take values from $[1..n]$. For $\lambda = 1$ the iteration is identical to an iteration of the \oea, but it is twice as costly.

\begin{algorithm}
    \textbf{input:} individual $x$\;
    \While{not terminated}{
        \tcp{Mutation phase}
        Choose $\ell \sim \Bin(n, \frac{\lambda}{n})$.
        \For{$i \in [1..\lambda]$}{
            $x_i \gets$ copy of $x$\;
            Flip $\ell$ bits in $x_i$, choosing positions u.a.r.\;
        }
        $x' \gets \argmax_{z \in \{x_1, \dots, x_\lambda\}}f(z)$\;
        \tcp{crossover phase}
        \For{$i \in [1..\lambda]$}{
            $y_i \gets$ copy of $x$\;
            For each bit in $y_i$ replace it with a corresponding bit from $x'$ with probability $\frac{1}{\lambda}$\;
        }
        $y \gets \argmax_{z \in \{y_1, \dots, y_\lambda\}}f(z)$\;
        \If{$f(y) \ge f(x)$}{
            $x \gets y$\;
        }
    }
    \textbf{return} $x$\;
    \caption{The \ollga maximizing a pseudo-Boolean function $f$~\cite{DoerrDE15}.}
    \label{alg:ollga}
\end{algorithm}

Although it is well known that the \ollga with a dynamic choice of $\lambda$ (via the one-fifth success rule~\cite{DoerrD18} or via a heavy-tailed distribution~\cite{AntipovBD22}) can have linear runtime on \onemax, we avoid using this algorithm in the portfolio. We aim to theoretically analyze a situation that can be met in practice, namely when the algorithm user has a limited portfolio of algorithms for their problem. Using algorithms with self-adaptive or random parameters essentially emulates having a large portfolio consisting of $n$ different algorithms. From the perspective of OAS, this would be a very large portfolio, and learning the efficiency of each algorithm might take a lot of time and be inefficient compared to even sub-optimal algorithms from the portfolio. For this reason, we restrict OAS to portfolios with the \ollga with static parameters. 

\subsection{Target Function}

In this paper, we consider the optimization of the classic theoretical benchmark called \onemax. This function is defined on bit strings of length $n$ and it returns the number of one-bits in its argument. More formally,
\begin{align*}
    \onemax(x) = \sum_{i = 1}^n x_i.
\end{align*}
There also exists a so-called generalized \onemax, which has some hidden bit string $z$ as a parameter and $\onemax(x)$ returns the number of bits which coincide in $x$ and $z$. In this paper, we consider only algorithms using unbiased variation operators (the operators that decide whether to flip a bit independently of the bit's position and value, see~\cite{LehreW12}), which run identically on generalized \onemax with any hidden bit string $z$. For this reason, to simplify the proofs, we consider the classic \onemax, where the hidden bit string is the all-ones bit string.

\subsection{Selection Mechanism}

In this paper, we consider a selection mechanism which starts with running the \olea and then can decide to switch to the \ollga. To simplify the analysis, we assume that it cannot switch back (similar to the study in~\cite{KostovskaJVNWED22}). This choice of the order of algorithms is natural, since it is known that mutation-based algorithms can be effective at hill-climbing, but closer to the optimum they slow down significantly due to the coupon-collector effect: it takes $\Theta(n\log n)$ evaluations to find the optimum when starting in distance $n^\varepsilon$ from it, where $\varepsilon > 0$ is any small constant. At the same time, in~\cite{AntipovBD25telo} it was shown the \ollga is much more effective in the late stages of optimization, when the algorithm is converging to the global optimum. In particular, they proved the following theorem.

\begin{theorem}[Theorem 3.1 in~\cite{AntipovBD25telo}]\label{thm:ollga-end}
    The expected runtime of the \ollga with static parameter~$\lambda$ (and mutation rate $p = \lambda/n$ and crossover bias $c = 1/\lambda$ as recommended in~\cite{DoerrDE15}) on \onemax with initialization at distance $D > 0$ from the optimum is $E[T_F] = O(\frac{n}{\lambda} \log D + D\lambda)$
    fitness evaluations. This is minimized by $\lambda = \sqrt{\frac{n\ln D}{D}}$, which gives a runtime guarantee of $E[T_F] = O\left(\sqrt{nD\ln D}\right)$.
\end{theorem}

The main goal of our analysis is to find the optimal time (or, more precisely, optimal distance from the optimum) to switch from the \olea to the \ollga and to estimate the runtime of the algorithm when it makes a switch in that distance. We do that in Section~\ref{sec:optimal-switch}. We also aim at finding a switching mechanism which would be able to detect that the algorithm is close to that optimal distance. We propose such a mechanism and show that it yields an asymptotically optimal runtime on \onemax in Section~\ref{sec:switching-moment}.

\section{Fixed-target Results}
\label{sec:fixed-target}

In this section, we show fixed-target results for the \olea on \onemax. Fixed-target analysis (see~\cite{BuzdalovDDV22}) aims at estimating the time that it takes an algorithm to find a solution of at least some fixed quality. This kind of analysis usually gives more details about the optimization process than the standard runtime analysis, since it reflects how hard different stages of optimization can be. However, as it was shown in~\cite{BuzdalovDDV22} and~\cite{VinokurovB22}, it is not trivial to find tight bounds for fixed-target runtime even for the simple \oea, mostly because of the algorithm being able to overshoot the target fitness value. The problem of overshooting is even sharper in the population-based EAs such as the \olea, for this reason we are not aware of any existing results for them.

For our purposes, however, it is enough to get asymptotical upper bounds on the fixed-target runtime of the \olea, and for this reason we state and prove the following theorem. We note that proving a similar result with a precise leading constant (which implies a matching lower bound) might be a valuable result itself, but it requires much more work to achieve it. Thus, we leave it for future research.

\begin{theorem}
    \label{thm:olea-fixed-target}
    The expected number of fitness evaluations that it takes the \olea optimizing \onemax to find a solution in distance at most $D$ from the optimum is at most
    \begin{align*}
        E[T_F] = \begin{cases}
            O\left(\frac{n\lambda\log^+\log^+\lambda}{\log^+\lambda}\right), &\text{ if } D \ge \frac{n}{\lambda}, \\
            O\left(\frac{n\lambda\log^+\log^+\lambda}{\log^+\lambda} + n\log\left(\frac{n}{\lambda (D + 1)}\right)^+\right), &\text{ if } D < \frac{n}{\lambda}, \\
        \end{cases}
    \end{align*} 
    where $(\cdot)^+$ denotes $\max\{ \cdot, 1\}.$
\end{theorem}
\begin{proof}
    We note that for $\lambda \le e^e$ these bounds are simplified to $O(n(1 + \log\frac{n}{D^+}))$, and they trivially follow from the bounds for the \oea from~\cite[Theorem~9]{BuzdalovDDV22} by noting that the progress in one iteration of the \olea dominates the progress of the \oea, while each iteration of the \olea with these values of $\lambda$ costs $\Theta(1)$ fitness evaluations. Hence, in the rest of the proof we assume that $\lambda > e^e$. 

    We use the classic fitness levels technique to prove this result~\cite{Wegener01,Wegener02}. We partition the set $[0..n]$ of all fitness values into subsets (levels)\footnote{In the classic fitness levels technique, it is the search space that is partitioned, not the fitness space. However, to ease the reading we simplify the notation and assume that any level in the \emph{classic} notation corresponds to all search points that have fitness values in the fitness level in \emph{our} notation.}, and for each level $i$ we estimate $p_i$---the probability that the \olea leaves level $i$ to go to a level with strictly better fitness in one iteration. Then the expected number of iterations the \olea spends in level $i$ is at most $\frac{1}{p_i}$, and the total runtime is at most the sum of these expected runtimes over all levels. The level partitions that we use below are inspired by the previous analyses of the \olea~\cite{DoerrK15,GiessenW17}, which distinguished similar phases depending on the current distance to the optimum. We only adapt their arguments to the fixed-target setting. We consider 3 cases based on the distance $D$. 

    \textbf{Large target distance:} $D \ge \frac{n}{\ln \lambda}$. We define levels as follows. $A_0$ denotes the optimal level (that is, we aim at estimating the runtime until this level is reached), hence $A_0 = [n - D..n]$. Then we define $\delta \coloneqq \lceil\frac{\ln\lambda}{2\ln\ln\lambda}\rceil$ (note that $\delta \ge 2$, since $\lambda \ge e^e$), and with this $\delta$ we define levels $A_i$ for $i \in [1..\lceil \frac{n - D}{\delta} \rceil]$ as 
    \begin{align*}
        A_i = [\max\{0, n - D - i\delta\}..n - D - (i - 1)\delta - 1].
    \end{align*}
    We say that the \olea is in level $i$, if the current individual $x$ has fitness $f(x) \in A_i$. Note that the fitness level can only decrease with time due to elitism of the \olea. We now estimate the probability $p_i$ that the algorithm leaves level $A_i$ in one iteration for an arbitrary $i$. We note that for this it is sufficient to have at least one of $\lambda$ offspring to be better than $x$ by at least $\delta$. Let $q_i$ be the probability that this happens in one particular offspring. Then $q_i$ is at least the probability that mutation flips exactly $\delta$ zero-bits and no one-bits. Denoting by $d$ the number of zero-bits in $x$ (the distance from it to the optimum), we have
    \begin{align*}
        q_i \ge \binom{d}{\delta} \left(\frac{1}{n}\right)^\delta \left(1 - \frac{1}{n}\right)^{n - \delta} \ge \frac{1}{e}\binom{d}{\delta} \left(\frac{1}{n}\right)^\delta.
    \end{align*}
    By the estimates of the binomial coefficients from~\cite{Doerr20bookchapter}, namely Eq.~(1.4.15), and since $d \ge D \ge \frac{n}{\ln\lambda}$, this is
    \begin{align*}
        &\ge \frac{1}{e} \left(\frac{d}{\delta n}\right)^\delta \ge \frac{1}{e} \left(\frac{1}{\delta\ln\lambda}\right)^\delta = \exp\left(\delta \ln\left(\frac{1}{\delta\ln\lambda} \right) - 1\right) \\
        &\ge \exp\left( - \frac{\ln\lambda}{2\ln\ln\lambda} \ln \left(\frac{\ln^2\lambda}{2\ln\ln\lambda}\right) - 1\right) \ge \exp\left( -\ln\lambda - 1 \right) = \frac{1}{\lambda e}.
    \end{align*}
    The probability that it happens in at least one of the individuals is therefore
    \begin{align}
        \label{eq:pi}
        p_i = 1 - \left(1 - q_i\right)^\lambda \ge 1 - \left(1 - \frac{1}{e\lambda}\right)^\lambda \ge 1 - e^{-\frac{1}{e}} = \Omega(1).
    \end{align}
    Hence, the expected time to leave each level is $O(1)$. There are at most $\frac{n}{\delta} = \Theta(\frac{n\log\log\lambda}{\log\lambda})$ levels, thus summing these expected times over all levels yields the bound of $E[T_I] = O(\frac{n\log^+\log^+\lambda}{\log^+\lambda})$ iterations. Since each iteration costs exactly $\lambda$ fitness evaluation (one per offspring), we obtain the desired bound $E[T_F] = O(\frac{n\lambda\log^+\log^+\lambda}{\log^+\lambda})$ on the expected number of evaluations for $D \ge \frac{n}{\ln\lambda}$.
    
    \textbf{Medium target distance:} $\frac{n}{\ln\lambda} > D \ge \frac{n}{\lambda}$. We now use the following partition into levels. The target level $A_0$ stays the same, that is, $A_0 = [n - D..n]$. For $i \in [1..\lfloor\frac{n}{\ln\lambda}\rfloor - D]$ we define $A_i = \{n - D - i\}$ (these levels consist of a single fitness value). For larger $i$ we define $A_i$ as a set of fitness values from $n - \lfloor\frac{n}{\ln\lambda}\rfloor - (i - \lfloor\frac{n}{\ln\lambda}\rfloor - D) \delta$ to $n - \lfloor\frac{n}{\ln\lambda}\rfloor - (i - 1 - \lfloor\frac{n}{\ln\lambda}\rfloor - D) \delta - 1$.
    In other words, these levels consist of $\delta$ consequent fitness values (where $\delta$ is the same as in the analysis for larger $D$). Figure~\ref{fig:levels} illustrates this partition.
    \begin{figure}
        \begin{center}
          \scalebox{0.75}{
            \begin{tikzpicture}
                \draw [ultra thick,->] (0,0) -- node [below=2pt, pos=1.0] {$f(x)$} (10, 0);

                \draw (9, 0.2) -- node [below, pos=1.0] {$n$} (9, -0.2);
                \draw (7.5, 0.2) -- node [below, pos=1.0] {$n-D$} (7.5, -0.2);
                \draw (6, 0.2) -- node [below, pos=1.0] {$n-\frac{n}{\ln\lambda}$} (6, -0.2);
                \draw (0, 0.2) -- node [below, pos=1.0] {$0$} (0, -0.2);
                \foreach \i in {1,...,5}
                    \draw (\i, 0.1) -- (\i, -0.1);
                \foreach \i in {1,...,4}
                    \draw (6 + \i * 0.3, 0.1) -- (6 + \i * 0.3, -0.1);

                \draw [decorate,decoration={brace,amplitude=20pt}] (0, 0) -- node [above=20pt, pos=0.5,align=center] {Levels with $\delta$\\fitness values} (6, 0);

                \draw [decorate,decoration={brace,amplitude=20pt}] (6, 0) -- node [above=20pt, pos=0.5,align=center] {Levels with a single\\fitness value} (7.5, 0);
                
                \draw [decorate,decoration={brace,amplitude=20pt}] (7.5, 0) -- (9, 0);

                \node (a) at (10, 1.2) {Target level $A_0$};
                \draw (a) -- (8.3, 0.8);
            \end{tikzpicture}
          }
        \end{center}
        \caption{Illustration of the fitness levels partition for small distances $D < \frac{n}{\ln \lambda}$.}
        \label{fig:levels}
    \end{figure}
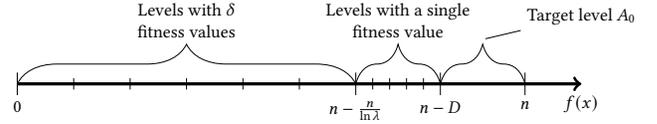

    We have already showed that the probability to leave the levels $A_i$ with $i > \lfloor\frac{n}{\ln\lambda}\rfloor - D$ (those that have $\delta$ different fitness values) is $\Omega(1)$. Now we show the same bound for the rest of the levels. The probability $q_i$ to create a better individual with mutation is at least the probability to flip one zero-bit, and no other bits. If the current individual $x$ is in distance $d$ from the optimum, it has $d$ zeros. Since we now consider case when $d > D \ge \frac{n}{\lambda}$, we have
    \begin{align*}
        q_i \ge \frac{d}{n}\left(1 - \frac{1}{n}\right)^{n-1} \ge \frac{1}{\lambda e},
    \end{align*} 
    hence by Eq.~\eqref{eq:pi} we have $p_i = \Omega(1)$, and the expected time until we leave the level is at most $O(1)$. The number of levels that we need to leave is at most $\lceil \frac{n}{\delta} \rceil + \lceil\frac{n}{\ln\lambda}\rceil - D = O(\frac{n}{\delta})$. Summation of the expected leaving times over all theses levels and multiplying by $\lambda$ (the cost of one iteration) results into bound $E[T_F] = O(\frac{n\lambda\log^+\log^+\lambda}{\log^+\lambda})$, which completes the proof of the lemma for $D \ge \frac{n}{\lambda}$.

    \textbf{Small target distance:} $D < \frac{n}{\lambda}$. We use the same level partition as in the previous case, but the main difference now is that in some levels $p_i$ is no longer $\Omega(1)$. The bounds from previous cases hold only for levels with fitness values at most $n - \frac{n}{\lambda}$, hence if $\lambda \ge n$, then the bound $E[T_F] = O(\frac{n\lambda\log^+\log^+\lambda}{\log^+\lambda})$ holds. Consider a fitness level $A_i$ with $i < \frac{n}{\lambda} - D$ (that is, for which the previous bound on $p_i$ does not hold). Let $d_i \coloneqq D + i$ denote the distance to the optimum of $x$ when the algorithm is in level $A_i$. Then we have
    \begin{align*}
        q_i \ge \frac{d_i}{n}\left(1 - \frac{1}{n}\right)^{n-1} \ge \frac{d_i}{en}.
    \end{align*}
    Using Lemma 2 from~\cite{AntipovBD22}, we obtain
    \begin{align*}
        p_i = 1 - (1 - q_i)^\lambda \ge \frac{\lambda q_i}{1 + \lambda q_i} \ge \frac{1}{2}\min\{1, \lambda q_i\} = \frac{\lambda d_i}{2en},
    \end{align*}
    hence the expected time to leave level $i$ is at most $\frac{2en}{d_i\lambda}$. Summing these runtimes over all sub-optimal levels $i$ with $i < \frac{n}{\lambda} - D$ (or, in other words, over all $d_i \in [D + 1..\lfloor\frac{n}{\lambda}\rfloor]$), we obtain
    \begin{align*}
        \sum_{d_i = D + 1}^{\lfloor\frac{n}{\lambda}\rfloor} \frac{2en}{\lambda d_i} &= \frac{2en}{\lambda} \sum_{d_i = D + 1}^{\lfloor\frac{n}{\lambda}\rfloor} \frac{1}{d_i} \le \frac{2en}{\lambda} \left(\ln\left(\frac{n}{\lambda (D + 1)}\right) + 1\right) \\
        &= O\left(\frac{n}{\lambda}\left(\log\left(\frac{n}{\lambda (D + 1)}\right) + 1 \right)\right),
    \end{align*}
    where in the penultimate step we used an estimate of a part of the harmonic series from~\cite[Lemma~2.4]{AntipovBD25telo}. Taking into account that this term is only added when $\frac{n}{\lambda} > 1$, multiplying this bound by $\lambda$ and adding the time needed to leave all levels with $i \ge \frac{n}{\lambda} - D$ (computed before), we obtain for all $D < \frac{n}{\lambda}$
    \begin{align*}
        E[T_F] &= O\left(\frac{n\lambda\log^+\log^+\lambda}{\log^+\lambda} + n\left(\log\left(\frac{n}{\lambda (D + 1)}\right)^+ + 1\right)\right) \\
        &= O\left(\frac{n\lambda\log\log\lambda}{\log\lambda} + n\log\left(\frac{n}{\lambda (D + 1)}\right)^+\right).\qedhere
    \end{align*}
\end{proof}

\section{Optimal Switch Time}
\label{sec:optimal-switch}

In this section, we bound the expected runtime of OAS which starts solving \onemax with running the \olea with population size $\lambda_1$ and then as soon as it finds an individual in distance $D$ from the optimum it switches to the \ollga with population size $\lambda_2$. Obviously, this is not a realistic assumption that an algorithm would know the distance to the optimum on any real-world problem, but these estimates will help us understand what an optimal switching mechanism would look like. The main result of this section is the following theorem.

\begin{theorem}
    \label{thm:known-switch-time}
    Let $D \in [0..n]$. Consider the OAS algorithm that solves \onemax by starting with running the \olea with population size $\lambda_1$ until it finds an individual in distance at most $D$ from the optimal all-ones bit string, and then it switches to the \ollga with population size $\lambda_2$. Then the expected time until it finds the optimum of \onemax is at most
    \begin{align*}
        E[T_F] &= O\left(\frac{n\lambda_1\log^+\log^+\lambda_1}{\log^+\lambda_1} + n\log\left(\frac{n}{\lambda_1 (D + 1)}\right)^+\right. \\
        &+ \left.\frac{n}{\lambda_2} \log D^+ + D\lambda_2 \right),
    \end{align*}
    where $(\cdot)^+$ denotes $\max\{\cdot, 1\}$.
\end{theorem}
This theorem trivially follows from Theorems~\ref{thm:ollga-end} and~\ref{thm:olea-fixed-target}. Since the expression given by this theorem is hard to interpret, in the rest of the section we give more explanations of how this bound looks in some particular cases.

First, consider a trivial case, when $\lambda_1 = 1$, that is, we start by running the \oea until we reach distance $D$, and then we switch to the \ollga. The following corollary describes this case.

\begin{corollary}
    \label{cor:oea-ollga}
    Consider the OAS algorithm optimizing \onemax by starting with running the \oea and then switching to the \ollga with population size $\lambda$ as soon as it finds an individual in distance at most $D$ from the optimum. Then the expected number of fitness evaluations until it finds the optimum of \onemax is at most
    \begin{align}
        \label{eq:etf-oea-ollga}
        E[T_F] = O\left(n\log\frac{n}{D^+} + \frac{n}{\lambda} \log D^+ + D \lambda\right).
    \end{align} 
    This bound is asymptotically minimized by choosing $D = \frac{n}{poly(\log n)} \cap O(\frac{n(\log\log n)^2}{\log n})$ and $\lambda = \Omega(\frac{\log n}{\log\log n}) \cap O(\frac{n}{D}\log\log n)$, which results in the expected runtime being $E[T_F] = O(n\log\log n)$. 
\end{corollary}
\begin{proof}
    The general expression~\eqref{eq:etf-oea-ollga} for the bound on $E[T_F]$ follows directly from Theorem~\ref{thm:known-switch-time} by taking $\lambda_1 = 1$ and $\lambda_2 = \lambda$. 
    
    To show that this bound cannot be smaller than $O(n\log\log n)$, we first note that for $D = 0$ the bound is $O(n\log n)$, hence we can ignore this case and write $D$ instead of $D^+$. Then we note that by Theorem~\ref{thm:ollga-end}, the last two terms are minimized by taking $\lambda = \sqrt{\frac{n\ln D}{D}}$, which results in $O(n\log\frac{n}{D} + \sqrt{nD\log D})$ bound. Since the first term is decreasing with $D$ and the second is increasing, this expression is minimized when they are asymptotically equal. This can be achieved by taking $D = \Theta(\frac{n (\log\log n)^2}{\log n})$, when both terms are  $\Theta(n\log\log n)$. Hence, the best bound that can follow from Eq.~\eqref{eq:etf-oea-ollga} is $O(n\log\log n)$.



    We now show that $O(n\log\log n)$ bound is achieved for the stated ranges of $D$ and $\lambda$ (since $D$ is super-constant, we omit the $D^+$ notation). The lower bound on $D$ implies
    \begin{align*}
        O\left(n\log\frac{n}{D}\right) &= O(n\log\log n).
    \end{align*}
    The upper bound on $D$ implies that $\log D = O(\log n)$. Together with the lower bound on $\lambda$ this implies
    \begin{align*}
        O\left(\frac{n}{\lambda}\log D\right) &= O\left(\frac{n\log\log n}{\log n} \log \frac{n\log\log n}{\log n}\right) = O(n \log\log n).
    \end{align*}
    The upper bound on $\lambda$ implies that
    \begin{align*}
        O(D\lambda) &= O\left(D \cdot \frac{n}{D} \log\log n\right) = O(n\log\log n).
    \end{align*}
    Hence, all three terms in Eq.~\eqref{eq:etf-oea-ollga} are $O(n\log\log n)$, which completes the proof.
\end{proof}

We note that the conditions on the values of $D$ and $\lambda$ are tight, that is, if they are not satisfied, then Eq.~\eqref{eq:etf-oea-ollga} gives a bound that is asymptotically larger than $O(n\log\log n)$ (but since this is just an upper bound, the actual runtime might be $O(n\log\log n)$ for a wider range of $D$ and $\lambda$). If $\lambda$ is $\omega(\frac{n}{D}\log\log n)$, then the term $D\lambda$ is $\omega(n\log\log n)$. If $D$ is $\frac{n}{\log^{\omega(1)} n}$, then the first term $n\log\frac{n}{D}$ is $\omega(n\log\log n)$. If $\lambda$ is $o(\frac{\log n}{\log\log n})$, then $\frac{n}{\lambda}\log D = \omega(\frac{n \log\log n}{\log n} \log D) = \omega(n \log\log n)$. Finally, if $D = \omega(\frac{n(\log\log n)^2}{\log n})$, then the upper bound on $\lambda$ is asymptotically smaller than the lower bound, hence we cannot choose $\lambda$ with which $D\lambda$ is $O(n\log\log n)$.

With Corollary~\ref{cor:oea-ollga} we prove that an OAS algorithm can asymptotically outperform the algorithms from its portfolio. The best expected runtime that the \oea and the \olea can achieve on \onemax is $\Theta(n\log n)$~\cite{LehreW12} and the best expected runtime of the \ollga with static parameters is $\Theta(n\sqrt{\frac{\log n \log\log\log n}{\log\log n}})$~\cite{DoerrD18}. OAS allows to achieve a runtime that is by a super-constant factor $\Omega(\sqrt{\frac{\log n \log\log\log n}{(\log\log n)^3}})$ better than the best runtime achievable by the \ollga using static parameters.

It is interesting to note that to satisfy conditions of Corollary~\ref{cor:oea-ollga}, the \ollga must use $\lambda \approx \Theta(\log n)$, while the best runtime of the \ollga on \onemax is achieved with much smaller $\lambda$, namely $\lambda \approx \Theta(\sqrt{\log n})$ (apart from sub-logarithmic factors). This is not surprising in the light of the results of~\cite{DoerrD18} who showed that the optimal fitness-dependent $\lambda$ is $\sqrt{n/d}$, where $d$ is the distance to the optimum, hence optimal $\lambda$ grows as the algorithm approaches the optimum. It is also in line with Theorem~\ref{thm:ollga-end}, which indicates that the smaller the starting distance is, the larger the value of $\lambda$ that minimizes the runtime.

To extend Corollary~\ref{cor:oea-ollga} to the setting when we start with the \olea with a non-trivial population size $\lambda > 1$, we note that the bound in Theorem~\ref{thm:olea-fixed-target} is the smallest when $\lambda = \Theta(1)$, and it is the same as for $\lambda = 1$. This bound is asymptotically larger when $\lambda = \omega(1)$. Hence, we cannot prove a better-than-$O(n\log\log n)$ bound for OAS with a portfolio consisting of the \olea and the \ollga. However, in the following corollary we show that this bound holds for a wide range of population sizes of the \olea. 
 
\begin{corollary}
  \label{cor:olea-ollga}
    Consider the OAS algorithm optimizing \onemax by starting with running the \olea with population size $\lambda_1$ and then switching to the \ollga with population size $\lambda_2$ as soon as it finds an individual at distance at most $D$ from the optimum. Assume that
    \begin{enumerate}
        \item[(a)] $\lambda_1 = O(\frac{\log\log n \log\log\log n}{\log\log\log\log n})$,
        \item[(b)] $D = \frac{n}{poly(\log n)} \cap O(\frac{n(\log\log n)^2}{\log n})$, and
        \item[(c)] $\lambda_2 = \Omega(\frac{\log n}{\log\log n}) \cap O(\frac{n}{D}\log\log n)$.
    \end{enumerate} 
    Then the expected number of fitness evaluations needed by the algorithm to find the optimum satisfies $E[T_F] = O(n\log\log n)$.
\end{corollary}

We omit the proof for reasons of space, and only note that it follows from combining the proof of Corollary~\ref{cor:oea-ollga} with putting the conditions on $\lambda_1$ into Theorem~\ref{thm:known-switch-time}.

\section{Automatically Detecting a Good Switching Moment}
\label{sec:switching-moment}

In the previous section we considered a case when the algorithm can decide to switch at a particular distance from the optimum. This is probably less relevant in practice, since it is rare when the algorithm has access to this information. However, the algorithm can estimate how hard the progress is. It can start with the \olea and after it does not observe fitness improvement for some particular time, it can switch to the more complicated mechanics of the \ollga, which allows faster convergence to the optimum. The problem here is that switching too early might make us use the \ollga at distances where it is not efficient due to the big cost of each of its iterations (when the \olea gives much better progress per fitness evaluation).

In this section, we propose a mechanism to detect the necessity of switching when our portfolio consists of the \oea\footnote{In order to simplify the proofs (and to get rid of two population sizes $\lambda_1$ and $\lambda_2$), we only consider the \olea with $\lambda = 1$ in this section. However, without proof we note that the results can be easily adapted for the \olea with $\lambda$ that satisfies the condition (a) of Corollary~\ref{cor:olea-ollga}.} and the \ollga with optimal population size $\lambda = \Theta(\log n)$, as suggested by Corollary~\ref{cor:oea-ollga}. This mechanism starts with the \oea and switches to the \ollga after $k$ consecutive iterations of the \oea without progress. The choice of $k$ plays an important role, as it defines the distribution of the distance at which we switch to the \ollga. Our intuition suggests that we should switch when we understand that the \oea is struggling to achieve progress in $2\lambda$ iterations, since at this point the cost of search equals that of the one iteration of the \ollga. However, it does not mean that we should choose $k = 2\lambda$: if we are unlucky in one of the fitness levels before reaching the optimal switching distance, this would increase the expected runtime. Instead, to be sure that the algorithm should indeed switch, this threshold should be set more conservatively as $k = \Theta(\lambda \log n)$. This allows us not to switch too early by making unfortunate events of $k$ failures in a row sufficiently unlikely before we reach the optimal interval of $D$. At the same time, it allows to switch to the \ollga with minimal delay from the optimal moment. 

To demonstrate the efficiency of this switching mechanism, we show that it also achieves $O(n\log\log n)$ runtime on \onemax, which is shown in the following theorem.

\begin{theorem}
    \label{thm:switch-mech}
    Consider optimization of \onemax with the OAS algorithm that starts with the \oea and then switches to the \ollga with population size $\lambda = \Theta(\log n)$ after the \oea does not manage to create a better individual for $k = \Theta(\lambda \log n)$ iterations in a row. The expected number of fitness evaluations it takes this OAS algorithm to find the optimum is $E[T_F] = O(n\log\log n)$.
\end{theorem}

Before we start the proof, we state and prove two auxiliary lemmas. The first lemma shows that we do not switch to the \ollga too early, and the second one shows that we do not switch to it too late.

\begin{lemma}
    \label{lem:no-early-switch}
    In the setting of Theorem~\ref{thm:switch-mech}, let $D \coloneqq \frac{2en\ln n}{k}$. Then we have $D = \Theta(\frac{n}{\log n})$ and the probability that the online algorithm selection mechanism switches to the \ollga before reaching distance $D$ from the optimum is at most $1/n$.
\end{lemma}

\begin{proof}
    First, since $\lambda = \Theta(\log n)$, we have $k = \Theta(\log^2 n)$, hence $D = \Theta(\frac{n}{\log n})$.

    The probability that we switch at some distance $d$ from the optimum is the probability that the \oea has $k$ unsuccessful iterations in a row when it is in this distance $d$. For any $d \ge D$ this probability is at most
    \begin{align*}
        \left(1 - \frac{d}{en}\right)^k &= \left(1 - \frac{d}{en}\right)^{\frac{en}{d} \cdot \frac{dk}{en}} \le \exp\left(-\frac{dk}{en}\right) \\
        &\le \exp\left(-\frac{2en\ln n}{k} \cdot \frac{k}{en}\right) = \frac{1}{n^2}.
    \end{align*}
    By the union bound, the probability that the switch occurs before the \oea reaches distance $D$ is at most the sum of these probabilities over all distances $d \ge D$, that is, it is at most $1/n$.
\end{proof}

\begin{lemma}
    \label{lem:no-late-switch}
    In the setting of Lemma~\ref{lem:no-early-switch}, after the \oea finds an individual in distance at most $D$ from the optimum for the first time, the expected time until it switches to the \ollga is at most $O(n\log\log n)$.
\end{lemma}
\begin{proof}
    Let $D' \coloneqq \frac{n}{\ln^2 n}$. Let $\tau$ be the time it takes the \oea that starts in distance at most $D$ from the optimum to find an individual that is in distance at most $D'$ from the optimum \emph{or} to make $k$ consequent unsuccessful iterations in a row. This time is dominated by the time it just takes the regular \oea to reach an individual in distance $D'$ from the optimum, which can be estimated via fitness levels technique, in the same way we did it in the proof of Theorem~\ref{thm:olea-fixed-target}. For each distance $d \in [D'..D]$ the probability $p_d$ to find an improvement is at least $d/(en)$, hence the total expected time it takes the \oea to reach distance $D'$ from distance $D$ (or smaller) is at most
    \begin{align*}
        E[\tau] &\le \sum_{d = D' + 1}^D \frac{1}{p_d} = en \sum_{d = D' + 1}^D \frac{1}{d} \le en \left(\ln \frac{D}{D'} + 1 \right) \\
        &= en\left(\ln(\Theta(\log n)) + 1\right) = \Theta(n \log\log n), 
    \end{align*}
    where we used an estimate of a part of the harmonic series from~\cite[Lemma~2.4]{AntipovBD25telo}, similar to the proof of Theorem~\ref{thm:olea-fixed-target}.  

    If after $\tau$ iterations the algorithm still has not switched to the \ollga, then its current individual $x$ is in distance $d \le D'$ from the optimum. For such $d$ the probability to have $k$ unsuccessful iterations in a row is at least the probability that mutation does not flip any of $d$ zero-bits in $k$ attempts, which is,
    \begin{align*}
        \left(1 - \frac{d}{n}\right)^k = \left(1 - \frac{d}{n}\right)^{\frac{n}{d} \cdot \frac{dk}{n}} = \left(e^{-1} - o(1)\right)^{O(1)} = \Theta(1),
    \end{align*}
    where we used $\frac{dk}{n} \le \frac{D'k}{n} = \frac{n}{\ln^2 n} \cdot \frac{\Theta(\log^2 n)}{n} = O(1)$.

    Consequently, at these distances $d \le D'$ the probability that the \oea performs $k$ unsuccessful iterations in a row before finding an improvement is at least constant. Hence, the expected number of distances the \oea visits before doing it (or finding the optimum) is also at most constant. In each distance it spends at most $k = \Theta(\log^2 n)$. Therefore, if the algorithm does not switch to the \ollga after $\tau$ iteration, it does it in $O(\log^2 n)$ iterations after that  (in expectation).

    In summary, the expected time it takes the OAS algorithm to switch to the \ollga after the \oea finds an individual in distance $D$ from the optimum is at most $E[\tau] + O(\log^2 n) = O(n\log\log n)$.
\end{proof}

With these two lemmas at hand, we prove Theorem~\ref{thm:switch-mech}.

\begin{proof}[Proof of Theorem~\ref{thm:switch-mech}]
    If the algorithm switches to the \ollga before reaching distance $D$, then the expected runtime is at most the sum of the expected runtimes of the \oea and the \ollga (each starting with a random solution), that is, $O(n\log n)$. Since by Lemma~\ref{lem:no-early-switch} this event happens only with probability $1/n$, its contribution to the expected total runtime is at most $O(\log n)$. 
    
    In the rest of the proof we condition on not switching before reaching distance $D$. We split the analysis into two parts: until the \oea reaches distance $D$ from the optimum and after that.

    \textbf{Part 1.} To prove the expected duration of this part, we can again use the fitness levels technique. We define a partition into levels, where each level corresponds to exactly one fitness value. For any particular fitness level $i$ the time that the \oea spends on this level conditional on not switching before reaching distance $D$ is dominated by the same unconditional time, since the condition implies that the algorithm does not spend more than $k$ iterations in this level. This means that this condition ``cuts'' the tail of the distribution of time, which leads to the domination. Therefore, the expected time it takes to leave a level which is in distance $d$ from the optimum is at most $\frac{en}{d}$. Summing this over all $d > D$, we have that the \oea reaches distance $D$ from the optimum in expected time that is at most
    \begin{align*}
        \sum_{d = D}^n \frac{en}{d} \le en \left(\ln\frac{n}{D} + 1\right) = en \ln(\Theta(\log n)) = \Theta(n \log\log n).
    \end{align*} 

    \textbf{Part 2.} After the algorithm reaches distance $D$, the condition that it has not switched to the \ollga does not affect the algorithm's behavior, and it spends the same number of iterations at each fitness value as it would if it just started in distance $D$ with the \oea. This allows us to apply Lemma~\ref{lem:no-late-switch}. Hence, the expected number of iterations that the \oea makes in this second part of optimization is at most $O(n \log\log n)$. Since we switch to the \ollga in distance at most $D$ from the optimum, by Theorem~\ref{thm:ollga-end} it takes it at most
    \begin{align*}
        O\left(\frac{n}{\lambda} \log D + D\lambda \right) = O\left(\frac{n}{\log n} \log n + \frac{n}{\log n} \log n\right) = O(n)
    \end{align*}
    expected fitness evaluations to find the optimum. Therefore, the second part of optimization takes at most $O(n \log\log n)$ time. 

    Combining the result for Parts 1 and 2, and also recalling the contribution of the runtime in case when the algorithm switches to the \ollga before reaching distance $D$, we conclude that the total runtime is at most $O(n \log\log n)$.
\end{proof}

We consider a very strict setting, where the OAS algorithm can only switch from the \oea to the \ollga once and has to proceed with the \ollga until the end of optimization. This restriction is justified by our goal---to show the capabilities of OAS even with such a small portfolio. However, allowing the OAS algorithm to switch back to the \oea (e.g., after the \ollga makes improvements $k'$ times in a row, which indicates that the progress is too easy to spend $2\lambda$ evaluations in each iteration) might allow to switch from the \oea to the \ollga more carelessly, for example after $k = \Theta(\lambda)$ failures of the \oea. Finding a good mechanism to switch back to the \oea and proving its effectiveness on a wider range of benchmarks than just \onemax would be a valuable complement of the results of this paper for practitioners.

\section{On a Hyper-heuristic Approach}
\label{sec:hh}

It is possible to extend the results of the previous section on the hyper-heuristic approach to switching between algorithms. Most of the previous studies of HHs involved only HHs that switch between different operators in the \oea. However, an iteration of the \olea or of the \ollga can be also considered as a more costly operator that creates an offspring, hence HH can also switch between algorithms. The main difference between the OAS approach and the HH approach is that the first aims at running each algorithm until it detects that it has become inefficient, while HHs usually iterate between operators using some strategy. Although this strategy can depend on the results of the previous iterations (which sometimes make HHs very similar to OAS), they usually decide which algorithm to use in every iteration following some pre-defined simple rules. This almost excludes different learning techniques that assess the efficiency of each operator (with exception being the HH studied in~\cite{DoerrLOW18}), but it also gives the HHs more flexibility and faster adaptivity (if their rules are chosen well).

We propose the following hyper-heuristic that at the start of each iteration chooses whether to run an iteration of the \oea or of the \ollga with population size $\lambda$. Initially it starts by running the \oea. In each next iteration it makes a decision as follows.
\begin{itemize}
    \item If in the previous $\lambda$ iterations there was no progress, it runs an iteration of the \ollga.
    \item Otherwise, it runs an iteration of the \oea.
\end{itemize}
In more simple words, in each fitness level this HH runs the \oea for $\lambda$ times or until it escapes this fitness level. If after that it is still on the same level, it starts performing iterations of the \ollga until it finds a better individual. The following theorem shows that this HH can also achieve the runtime of $O(n\log\log n)$ on \onemax.

\begin{theorem}
    The expected runtime of the described HH with $\lambda = \Theta(\log n)$ on \onemax is $O(n\log\log n)$.
\end{theorem}

\begin{proof}
    We will first show that this HH reaches distance $d \le \frac{n}{\ln^2 n}$ from the optimum in $O(n\log\log n)$ expected time, and in the rest of optimization it spends at most $O(n)$ more fitness evaluations.

    First, assume that the current distance to the optimum $d$ is at least $\frac{n}{\ln^2 n}$ and let $t_d$ be the expected number of fitness evaluations spent by the HH in this distance $d$ before it finds progress. To estimate $E[t_d]$, we will use slightly modified amortized analysis~\cite{Tarjan85}. First, let $X$ be the number of fitness evaluations performed by the \oea in this distance $d$. Let $Y$ be the number of iterations that the \ollga makes to find an improvement starting in distance $d$ from the optimum (conditional on that the \oea did not find an improvement). Then if we define the coin value as $Y' = 1 + 2Y$, then after running the \oea for at most $\lambda$ iterations, even if it does not achieve progress, we have enough budget to do so with the \ollga. Hence, the total cost of finding an improvement at distance $d$ is at most $X \cdot Y'$. To estimate $E[X \cdot Y']$ we note that $X$ and $Y'$ are independent, since by definition $Y'$ does not depend on whether the \oea achieved progress in its $\lambda$ attempts. Hence, $E[X \cdot Y'] = E[X]E[Y'] = E[X](1 + 2E[Y])$.
    
    Since the probability that the \oea finds an improvement in one iteration is $\frac{d}{en}$ (see the  proof of Theorem~\ref{thm:olea-fixed-target}), we have $X \sim \min\{\Geom(\frac{d}{en}), \lambda\}$, and hence $E[X] \le \frac{en}{d}$. By~\cite[Lemma~7]{DoerrDE15}, for some constant $C$, the probability that the \ollga finds an improvement in one iteration is at least
    \begin{align*}
        C\left(1 - \left(1 - \frac{1}{\ln^2 n}\right)^{\lambda^2/2}\right) = C (1 - e^{-\Omega(1)}) \ge C'
    \end{align*}
    for some constant $C'$ that depends on the value of $\lambda$. Hence, $E[Y] \le C'^{-1}$, and the expected time spent in distance $d$ from the optimum is at most $\frac{en}{d} (1 + 2C'^{-1})$. Summing up these times over all distances $d \ge \frac{n}{\ln^2 n}$ we obtain that the expected time until the HH finds an individual in distance $d < \frac{n}{\ln^2 n}$ is at most
    \begin{align*}
        \sum_{d = \lceil \frac{n}{\ln^2 n} \rceil}^n \frac{en}{d} (1 + 2C'^{-1}) = O\left(n \log\left(\frac{n}{n/\ln^2 n}\right) \right) = O(n\log\log n).
    \end{align*}
    
    In smaller distances $d < \frac{n}{\ln^2 n}$, even if we pessimistically assume that \oea never finds an improvement, the total time spent on the \ollga is at most $O(n)$, as we showed in the proof of Theorem~\ref{thm:switch-mech}. The total number of evaluations wasted on the \oea on these small distances is at most $\lambda \frac{n}{\log^2 n} = o(n)$. Hence, the total runtime is at most $O(n\log\log n)$.
\end{proof}


We note that in contrast with the setting of our OAS algorithm, this HH can switch to the \ollga after just $\lambda$ unsuccessful iterations of the \oea instead of $\Theta(\lambda \log n)$ that is required by the OAS algorithm. This is because the HH can switch back to the \oea after each successful iteration, which allows it to ignore unlucky events when it switches to the \ollga early. As a result of restricting the OAS algorithm from switching back to the \oea, it has to switch only when it is really sure that the \ollga will be more effective. Although it leads to some insignificant lag in switching, it avoids wasting fitness evaluations on the \oea iterations in the late stages of optimization. However, both these downsides do not affect our upper bounds on the runtime of both the HH and the OAS algorithm.  

\section{Discussion and Conclusion}

In this paper, we have demonstrated that an OAS algorithm can outperform the algorithms from its portfolio, even when the portfolio consists of only two algorithms. We proposed a mechanism to switch between them that is based on being sure that one of the algorithms is not effective anymore. This mechanism allows to solve \onemax in $O(n\log\log n)$ expected time when using the \oea and the \ollga: the performance that is asymptotically superior to any of these two algorithms working alone, even with the best tuned (but static) parameters. We also proposed a hyper-heuristic based on the same algorithms which shows the same performance gain.

Seeing this upper bound that is better than the $\Theta(n\log n)$ runtime of algorithms from the portfolio (or $\approx\Theta(n\sqrt{\log n})$, if we allow the \ollga to use differently tuned parameters) naturally raises the question what is the best runtime that we can achieve using $m$ different algorithms in the portfolio. Trivially, using $m \approx \sqrt{n}$ is enough to be able to simulate the self-adaptive \ollga that can solve \onemax in linear time, but what is the minimal portfolio size that achieves this runtime? And what is the best performance gain we can get by adding algorithms to the portfolio? This question is important in real-world applications, since, on the one hand, the practitioners want to use a sufficiently large portfolios of algorithms that would allow the best possible performance in any situation, but, on the other hand, large portfolios make it significantly more difficult to quickly (re-)learn the effectiveness of each algorithm without spending too much time. Addressing these questions in further theoretical research would be extremely useful for the algorithm users. From practical perspective, it is also important to verify if the proposed OAS methods can be used on practical problems with more complex landscapes. 

One of the essential elements of our analysis was the fixed-target results shown in Section~\ref{sec:fixed-target}. Together with the fixed-start results from~\cite{AntipovBD25telo} they provided us enough information about behavior of each algorithm at the beginning and at the end of optimization. For analysis of larger portfolios it is necessary to extend these tools to be able to prove ``fixed-start-fixed-target'' results. This seems to be a trivial task for asymptotical upper bound as the ones we showed in Theorem~\ref{thm:olea-fixed-target}, but proving the tight bounds with leading constants might be much more challenging, especially for population-based EAs. However, the outcome of such results will give us a solid foundation for analysis of more complex OAS.

Finally, we address the tightness of our $O(n\log\log n)$ bound shown for multiple settings. Although we do not use any tools that provide lower bounds, we note that our fitness-level arguments used in Theorem~\ref{thm:olea-fixed-target} are such that it is unlikely that the algorithm skips a super-constant fraction of levels, hence following the arguments of the fitness-levels method for lower bounds~\cite{DoerrK24} it is natural to say that our bound should be tight. The same argument suggests that the bounds from~\cite{AntipovBD25telo} that we use are also asymptotically tight. Hence, our bound should also be tight. However, despite such simple intuition, proving it would require more effort, and we leave it for further studies.

\begin{acks}
We acknowledge funding by the European Union (ERC, ``dynaBBO'', grant no.~101125586). This research was also jointly funded by the French National Research Agency (ANR-23-CE23-0035) and the German Research Foundation (DFG; LI 2801/7-1), through project \textsc{Opt4DAC}.
\end{acks}

\bibliographystyle{ACM-Reference-Format}
\bibliography{bibliography}

\end{document}